%% file: main.tex
\pdfoutput=1
\documentclass[letterpaper, 10 pt, conference]{ieeeconf}
\IEEEoverridecommandlockouts
\usepackage{cite}
\usepackage{amsmath,amssymb,amsfonts}

\usepackage{amsthm}
\usepackage{graphicx}
\usepackage{textcomp}
\usepackage{url}
\usepackage{xcolor}
\usepackage{booktabs}
\usepackage{multirow}
\usepackage{array}
\input{numbers}
\usepackage{hyperref}
\hypersetup{colorlinks=true, linkcolor=black, citecolor=black, urlcolor=blue, pdftitle={Structure over Reasoning: Unified Zero-Shot Navigation with a Small Vision-Language Model}, pdfauthor={Zeqi Li, Xunyi Zhao, Zerui Li, Sihao Lin, Qi Wu, ShiJie Li}}
\makeatletter
\renewenvironment{proof}[1][\proofname]{\par\pushQED{\qed}\normalfont\topsep2\p@\@plus2\p@\relax\trivlist\item[\hskip\labelsep\itshape#1\@addpunct{.}]\ignorespaces}{\popQED\endtrivlist\@endpefalse}
\makeatother
\newtheoremstyle{tight}{2pt plus 1pt}{2pt plus 1pt}{\itshape}{}{\bfseries}{.}{ }{}
\theoremstyle{tight}
\newtheorem{proposition}{Proposition}
\DeclareMathOperator*{\argmax}{arg\,max}

\newlength{\topfloatshim}
\begin{document}
\bstctlcite{BSTcontrol}

\title{Structure over Reasoning: Unified Zero-Shot\\Navigation with a Small Vision-Language Model}

\author{Zeqi Li$^{1}$, Xunyi Zhao$^{2}$, Zerui Li$^{2}$, Sihao Lin$^{2}$, Qi Wu$^{2}$ and ShiJie Li$^{3}$%
\thanks{This work has been submitted to the IEEE for possible publication. Copyright may be transferred without notice, after which this version may no longer be accessible.}%
\thanks{$^{1}$Technical University of Munich, Germany.}%
\thanks{$^{2}$Adelaide University, Responsible AI Research Centre, Australian Institute for Machine Learning, Australia.}%
\thanks{$^{3}$Institute for Infocomm Research (I2R), A*STAR, Singapore.}%
}

\maketitle

\begin{abstract}
Vision-language navigation is moving into photorealistic, physically simulated scenes, where the leading zero-shot agents put a vision-language model in the control loop to choose every move and stop. Robots favour small models, whose abilities are uneven: Qwen3-VL-4B can judge whether a single view looks like the goal, but its route reasoning is at chance, its distance estimates nearly constant, its stop votes unreliable. Our agent therefore asks it two perceptual questions per step, which direction looks like the goal and whether the view justifies stopping, plus one gated frontier query on route-bearing instructions, and treats the answers as evidence, not actions: a route-blind semantic field over a bird's-eye map accumulates the directional scores through a tempered update, a skeleton graph supplies frontiers and geodesic distances, and a two-stage terminator turns stop votes into termini through a trail-consistency gate and a budget-forced return. With every baseline re-run in one installation at one budget, the agent exceeds all four zero-shot baselines of the VLNVerse subset in mean success on both granularities, though the paired interval against the strongest coarse-grained baseline includes zero. Termination is the dominant lever in ablation, thirteen attempts to give the model more of the navigation establish no reliable gain, and the agent runs on a Unitree Go2 quadruped with off-board inference. Code, per-episode logs and a robot video will be released at \href{https://anonymous.4open.science/r/structure-over-reasoning-PAGE}{https://anonymous.4open.science/r/structure-over-reasoning-PAGE}.
\end{abstract}

\section{Introduction}

Vision-language navigation asks an agent to follow a natural-language instruction to a goal it has never seen. The field is moving from panoramic graphs into photorealistic, physically simulated scenes such as VLNVerse \cite{vlnverse2025}, where the leading zero-shot agents place a vision-language model in the control loop and let it decide each move and each stop from panorama, history and full instruction \cite{zhou2024navgpt,chen2024mapgpt,long2024discussnav,qiao2025opennav}. The recipe needs no training, but it assumes a model large enough to reason about a whole route from a few images and to judge, from a single view, whether the goal has been reached.

A robot cannot always afford one: a cloud model adds latency and recurring cost, an onboard one must fit limited memory and compute. Qwen3-VL-4B \cite{qwen3vl2025} occupies \cWeightsGB{} GB and answers a step of our agent in \cOursStepS{} seconds on an \cGPU{}, against \cBaseStepS{} seconds when asked to reason at length. A smaller model does not make every part of navigation slightly worse; some parts fail while others hold. Shown one view and the goal phrase, the 4B model can usually tell whether that direction looks like the goal; across the variants we tried, its route reasoning over candidate directions stays at chance, its distance estimates are nearly constant, its grounding finds the right category but not the instance the instruction names, and its fine-grained stop votes fall within three metres of the goal less than a third of the time. A small-model navigator should use what stays reliable and move the rest into structure the model only feeds.

Our agent asks two narrow perceptual questions, which of twelve directions looks like the goal and whether the view justifies stopping, and treats the answers as evidence and proposals, not actions. A bird's-eye occupancy map built from depth supports a skeleton graph whose frontier nodes are the only sub-goals the agent adopts, the moves being chosen among the benchmark's waypoint candidates, and a route-blind semantic field over the same map accumulates the directional scores with a tempered update, so that the field, not the model, remembers where the goal probably is. When a fine-grained instruction carries route clauses the model may also propose a frontier through a chain-of-thought query \cite{wei2022cot}, gated and coupled conservatively; on coarse-grained instructions, which name only the destination, or without a usable proposal, the field drives alone. A stop vote is likewise checked against the trail's evidence rather than executed, and an episode nearing its budget without an accepted vote walks back to the best position seen. One policy with a single granularity-specific parameter serves both.

Among the components removed in our ablations, termination produces the largest drop in success. On fine-grained instructions an agent that never stops and walks the full forty-step budget of our ablations passes within three metres of the goal in \aBoneOSR{}\% of episodes but succeeds in \aBoneSR{}\%; the full terminator raises this to \aBzeroSR{}\% in about half the steps. Reaching is not solved either, since even that agent stays beyond three metres in more than a fifth of episodes, but the gap shows that nearing a good position and ending there are separate problems. The ablations agree elsewhere: the field and the route channel contribute nearly additively, and the terminator is the largest single lever. Thirteen attempts to hand responsibility back to the model, from longer route reasoning and sub-goal decomposition to grounding-based final approach and model-selected return positions, establish no reliable gain, and two lose beyond the run-to-run spread. Throughout, the model supplies local evidence and structure mediates it.

The contributions are fourfold. First, a decomposition of zero-shot navigation around a small vision-language model into a route-blind semantic field, a skeleton graph and a gated route channel, with the field update, the coupling and the action utility in closed form and one policy for both granularities. Second, a two-stage terminator with a trail-consistency gate and a budget-forced return proved affordable under the benchmark's displacement protocol, with an ablation that quantifies each stage. Third, on the official VLNVerse subset, with every zero-shot baseline re-run in one installation at the same thirty-step budget, the agent exceeds the strongest baseline in mean success by \oFineThirtyGain{} points on fine-grained instructions, with lower navigation error and higher path efficiency, and by \oCoarseGain{} on coarse-grained ones, all three-run means, and runs on a Unitree Go2 quadruped. Fourth, ablations and thirteen attempts to give the model more of the navigation, mapping what this model and prompt family support reliably and what still needs explicit structure.

\section{Related Work}

\textbf{Vision-language navigation.} The task was introduced on panoramic graphs of real buildings \cite{anderson2018r2r} and extended to finer-grained instructions \cite{ku2020rxr}, object-centric goals \cite{qi2020reverie} and continuous environments \cite{krantz2020vlnce}, where trained agents pair a waypoint predictor with a learned policy \cite{krantz2021waypoint,an2024etpnav,an2023bevbert} or train a vision-language-action model end to end \cite{zhang2024navid,cheng2025navila}. VLNVerse \cite{vlnverse2025} moves the task into a physics-based simulator with real-scale scenes and two instruction granularities; we adopt its official evaluation subset and protocol without change.

\textbf{Zero-shot navigation with language models.} A line of work replaces the trained policy by a language model that reads a textual description of the panorama and the history and chooses among candidate directions \cite{zhou2024navgpt,long2024discussnav,lin2025navcot}, with a map-like textual memory in MapGPT \cite{chen2024mapgpt}, open-source models in continuous environments in Open-Nav \cite{qiao2025opennav} and sub-task decomposition in InstructNav \cite{long2024instructnav}. For object-goal search, language models score frontiers \cite{zhou2023esc,yu2023l3mvn,yin2024sgnav} and vision-language embeddings build value maps over the explored area \cite{yokoyama2024vlfm}; VoroNav reasons over a reduced Voronoi graph with captions of far-sighted views \cite{wu2024voronav}, which our skeleton graph and frontier captions resemble. Two things separate our agent from this family: the model never chooses a waypoint, since it scores sectors and proposes frontiers while a belief over the map decides where to go, and it never decides to stop on its own, since its vote is verified against the trail; Section~\mbox{\ref{sec:negative}} shows that with a 4B model both restrictions are what make the agent work.

\textbf{Semantic maps and stopping.} Occupancy grids \cite{elfes1989occupancy,thrun2005probabilistic}, frontier-based exploration \cite{yamauchi1997frontier} and semantic exploration policies \cite{chaplot2020semexp} are the classical basis of our map, and open-vocabulary maps \cite{huang2023vlmaps} show that language-conditioned scores can be written into spatial memory; our field differs in being a probability distribution over reachable cells, updated by tempered exponential evidence accumulation and blind to the route. When to stop is a known weakness of VLN agents, addressed by progress monitors \cite{ma2019selfmonitoring} and backtracking \cite{ke2019tacticalrewind} on the panoramic graph; the VLNVerse baselines execute the model's stop directly \cite{vlnverse2025}. We keep the vote as a proposal, add two cheap verifications and quantify each.

\section{Method}
\label{sec:method}

\textbf{Task formulation.} VLNVerse places a cylindrical agent in physically simulated interior scenes rendered at real scale in Isaac Sim 6.0.1 and evaluates instruction following at two granularities: fine-grained instructions describe the route clause by clause and end with the goal; coarse-grained instructions name the goal object and its surroundings without a route. At each step the agent observes a twelve-view RGB and depth panorama and its pose, and moves by selecting a waypoint from the toolkit's waypoint predictor. We use the benchmark's Tel-Hop protocol: the agent is teleported to the chosen waypoint, and only the terminus is checked against the physical scene. An episode ends when the agent stops or its budget of $T$ steps is exhausted, and succeeds if the final position is within $r = 3$ m of the goal.

\textbf{Overview.} The design rests on one division of labour: the model answers local questions, structure handles the global problem. Shown a single view and the goal phrase, Qwen3-VL-4B can say whether that direction looks like the goal and whether the view justifies stopping; asked which candidate direction continues a route, it answers at chance level (Section~\ref{sec:negative}). The agent therefore asks only the local questions and writes the answers into a semantic field over a bird's-eye map that is route-blind by construction, so that it cannot become a second, failing route reasoner. Where to go next and when to end the episode are decided from geometry and bookkeeping: a skeleton graph supplies frontiers and distances, a utility over the waypoint candidates picks the move, and a terminator verifies each stop vote against the trail (Fig.~\ref{fig:arch}).

\begin{figure*}[t]
\vspace*{\topfloatshim}
\centering
\includegraphics[width=0.94\textwidth]{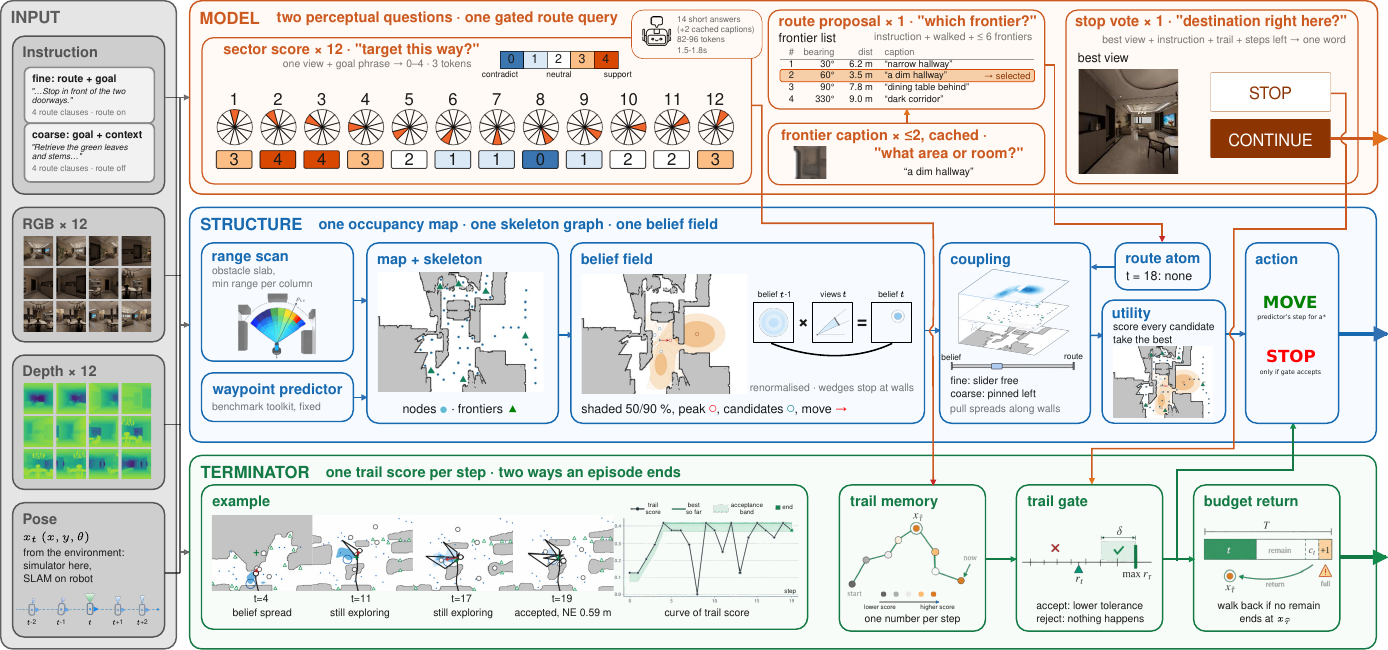}
\caption{One navigation step in four lanes. \textbf{Inputs} (grey): the instruction, twelve RGB views, the matching depth views and the pose. \textbf{The model} (orange): a five-point sector score per view, a stop vote on the best-scoring view and, on route-bearing instructions, a gated chain-of-thought query over the frontier list. \textbf{Structure} (blue): depth becomes a range scan and an occupancy map, the map a skeleton graph with frontier nodes, and the sector scores a route-blind belief field; the proposal enters the field through a gate closed on coarse-grained instructions, and a utility over the waypoint candidates picks the move. \textbf{The terminator} (green): a vote is accepted only within a tolerance $\delta$ of the best trail score; otherwise the agent walks back to that position when the budget runs short. The model, structure and terminator panels and the strip below them are real data from one fine-grained episode.}
\label{fig:arch}
\end{figure*}

\subsection{Bird's-Eye Map and Skeleton Graph}

In each of the twelve depth views, which span the full circle with a ninety-degree horizontal field of view each, every column is reduced to the nearest surface whose height lies within an obstacle slab around the camera, which gives a range scan around the agent, with a column that sees only floor or ceiling counted as free to the trusted range, and the scan is integrated into an occupancy grid \cite{elfes1989occupancy,thrun2005probabilistic} with a trusted range $R$. Free cells reachable from the agent's position $x_t$ form the support $S_t$ of everything that follows, and $d_t(x, y)$ denotes the geodesic distance between two cells of $S_t$, infinite when no free path connects them. The geodesic ball $B_t(x) = \{\, y \in S_t : d_t(x, y) \le r \,\}$ uses the benchmark's success radius $r = 3$ m, so that every readout of the field measures mass that would count as success. A skeleton is extracted from the free space, in the spirit of the reduced Voronoi graph of VoroNav \cite{wu2024voronav}, and yields a graph whose nodes are corridor and room centres, with median edge length $\ell_t$, and whose frontier nodes $F_t$ lie on the boundary between known free space and the unknown \cite{yamauchi1997frontier}. Local motion candidates $a \in A_t$, the actual steps of one to a few metres to a position $x_a$, are produced by the waypoint predictor shipped with the benchmark toolkit; the agent only chooses among them.

\subsection{Route-Blind Semantic Field}

The field is a belief $b_t$ over $S_t$ about where the goal is, uniform at the start, $b_0 = 1/|S_0|$. At each step the model is shown each of the twelve views in turn and asked how likely the goal phrase lies in that direction, visible or further along, on a five-point scale from zero to four. The answer $\sigma_{t,v}$ for view $v$ is centred to $s_{t,v} = (\sigma_{t,v} - 2)/2 \in [-1, 1]$, so that a score below two is evidence against the direction, and projected into the map as an evidence image over the thirty-degree wedge $W_{t,v}$ of that view,
\begin{equation}
\begin{aligned}
s_t(x) &= \sum_{v=1}^{12} s_{t,v}\, \mathbf{1}\!\left[x \in W_{t,v}\right], \\
W_{t,v} &= \{\, x \in S_t : \theta_t(x) \in \Theta_v,\ \|x - x_t\| \le \rho_{t,v} \,\},
\end{aligned}
\label{eq:evidence}
\end{equation}
where $\theta_t(x)$ is the bearing of $x$ from the agent, $\Theta_v$ the angular range of view $v$ and $\rho_{t,v}$ the free range the depth scan observed in that view, so that evidence never crosses a wall. The prompt also asks for a rough distance, but the answers are nearly constant, so the evidence is spread over the visible wedge. The belief is updated in the log domain, as in log-odds occupancy mapping \cite{thrun2005probabilistic}, and then mixed with a small uniform leak. The five-point scores are not calibrated likelihoods, so the update is exponential evidence accumulation rather than a posterior, and $b_t$ is a belief in that sense only:
\begin{equation}
\begin{aligned}
\tilde b_t(x) &= \frac{b_{t-1}(x)\, e^{\beta_t s_t(x)}}{\sum_{y \in S_t} b_{t-1}(y)\, e^{\beta_t s_t(y)}}, \\
b_t &= (1 - \nu)\, \tilde b_t + \frac{\nu}{|S_t|},
\end{aligned}
\label{eq:field}
\end{equation}
with a small leak $\nu$. The leak keeps every region recoverable, since the coupling below can only redistribute mass. The evidence strength $\beta_t = \beta_0\, c^{\mathrm{p}}_t\, 2^{-n_t / h}$ tempers the update: $\beta_0$ is the base strength, $c^{\mathrm{p}}_t \in [0, 1]$ the fraction of the twelve views for which the model returned a parseable score, $n_t$ the number of earlier steps taken from the current map cell, and $h$ the number of revisits after which the strength halves, so that a second look from the same place does not count as a second independent observation. Cells that enter $S_t$ as the map grows are initialised at the mean of the previous belief, $1/|S_{t-1}|$, which keeps the normalised entropy $\bar H_t = -\sum_{x \in S_t} b_t(x) \log b_t(x) / \log |S_t|$ comparable across steps as map grows; $\bar H_t$ weights field's voice in action utility below.

The field is route-blind by construction: the prompt carries only the goal phrase, never a route clause, and a guard rejects any phrase with directional vocabulary.

\subsection{Route Channel and Gated Coupling}

Fine-grained instructions describe a route before naming the goal. The instruction is parsed into clauses, of which $k_r$ are route-bearing and one is terminal, and the route share $\pi_r = k_r / (k_r + 1)$ is fixed for the episode. A route channel proposes a frontier at each step with a chain-of-thought query \cite{wei2022cot} over the skeleton, in the manner of map-prompted navigation \cite{chen2024mapgpt}. The model receives the instruction, the steps and metres walked so far and a numbered list of the unexplored frontiers ordered by bearing, each described by its clock bearing, its distance and a cached one-phrase caption of the view in that direction as in VoroNav \cite{wu2024voronav}, and answers with the index of the frontier the route most likely continues toward, or zero if the destination is probably within a couple of metres; the proposal reaches the belief only when $k_r > 0$, through the gate below. The chosen frontier is the step's proposal, indexed $\rho$ and located at cell $x_\rho$, with confidence $c^\rho_t = 1$; a silent answer, a malformed one or the answer zero gives $c^\rho_t = 0$. The proposal conditions the field on a geodesic kernel around that cell,
\begin{equation}
\begin{aligned}
K_\rho(x) &= e^{-d_t(x, x_\rho)^2 / 2 \sigma_t^2}, \\
p_\rho(x) &= \frac{b_t(x)\, K_\rho(x)}{\sum_{y \in S_t} b_t(y)\, K_\rho(y)},
\end{aligned}
\label{eq:kernel}
\end{equation}
with bandwidth $\sigma_t^2 = \sigma_f^2 + \ell_t^2$, a floor $\sigma_f$ added to the median edge length, and $\sigma_t$ capped at the trusted range $R$. Because $p_\rho$ is the field itself reweighted, the injection moves mass toward the proposal without inventing mass where the field has none. The coupled belief is the contraction
\begin{equation}
\begin{aligned}
\hat b_t &= (1 - \tilde\lambda_t)\, b_t + \tilde\lambda_t\, p_\rho , \\
\tilde\lambda_t &= \lambda_{\max}\, \pi_r\, c^\rho_t ,
\end{aligned}
\label{eq:couple}
\end{equation}
with $\lambda_{\max} < 1$, so that even a confident wrong proposal leaves a fixed share of the belief with the route-blind field. When the instruction has no route clause, as in coarse-grained instructions, $\pi_r = 0$ and the gate is closed, and when the model declines to answer, $c^\rho_t = 0$; in both cases the field alone drives, so the same agent serves both granularities, with commitment penalty the only parameter set per granularity.

\textbf{Action utility.} Each candidate move $a \in A_t$ is scored by
\begin{equation}
\begin{aligned}
U_t(a) &= w^g_t \sum_{x \in B_t(x_a)} \hat b_t(x)
 + w^r_t\, e^{-d_t(x_a, x_\rho)/\gamma} \\
 &+ \mu\,(1 - \tilde\lambda_t)\, f_t(a)
 - \eta\, d_t(x_t, x_a) / R \\
 &- \kappa\, \mathbf{1}\!\left[\, \angle(\theta_a, \theta_{t-1}) > \pi/2 \,\right],
\end{aligned}
\label{eq:utility}
\end{equation}
and the agent takes $\argmax_a U_t(a)$. The five terms are the coupled belief mass within the success radius of the candidate; a pull toward the proposed frontier that decays with geodesic distance on a scale $\gamma$; an exploration bonus $f_t(a)$, the fraction of unknown cells in a disc of radius $r_e$ around the candidate, scaled by $1 - \tilde\lambda_t$ with weight $\mu$ so that frontiers the coupling already favours are not rewarded twice; the travel cost with weight $\eta$, normalised by the trusted range; and a commitment penalty $\kappa$ on reversals. The first two weights share the two channels' confidences, $w^g_t = u^g_t / (u^g_t + u^r_t)$ and $w^r_t = u^r_t / (u^g_t + u^r_t)$ with $u^g_t = 1 - \bar H_t$ and $u^r_t = \omega\, \pi_r c^\rho_t$, so a sharp field speaks louder than a flat one; without a valid proposal, $c^\rho_t = 0$, the route term is dropped, $w^r_t = 0$ and $w^g_t = 1$. The commitment penalty matters more than its size suggests. A candidate whose bearing $\theta_a$ differs from the bearing $\theta_{t-1}$ of the last move by more than ninety degrees, $\angle(\cdot,\cdot)$ denoting the difference taken on the circle, pays the fixed penalty $\kappa$ on fine-grained instructions; on coarse-grained instructions, where no route is followed, $\kappa = 0$. This commitment hysteresis suppresses the oscillation of an agent that re-evaluates every frontier every step.

\subsection{Two-Stage Terminator}
\label{sec:terminator}

Stopping is modelled as an action the model proposes and the trail verifies. At each step the model is asked, with the best-scoring view, the instruction and a per-step record of its recent movement in view, whether the destination is right there so that stopping here completes the task, and answers with one word. A stop vote is not executed directly. The agent keeps the trail score
\begin{equation}
r_t = \frac{1}{12} \sum_{v=1}^{12} s_{t,v},
\label{eq:trail}
\end{equation}
the mean centred sector score the field received at each visited position, and accepts the vote only if the current position is at least as good as any position on the trail up to a tolerance $\delta$:
\begin{equation}
r_t \;\ge\; \max_{\tau \le t} r_\tau \;-\; \delta , \qquad t \ge t_0 ,
\label{eq:gate}
\end{equation}
a rejected vote leaves the agent exploring, and votes in the first $t_0$ steps are ignored, where the trail is too short for the gate to mean anything.

The second stage guarantees a deliberate terminus for episodes without an accepted vote, a budgeted form of backtracking \cite{ke2019tacticalrewind}. Let $\hat\tau_t = \argmax_{\tau \le t} r_\tau$ be the step with the best trail score, $R_t$ the longest step the waypoint predictor has offered so far, and
\begin{equation}
c_t = \left\lceil \frac{\|x_t - x_{\hat\tau_t}\|}{R_t} \right\rceil
\label{eq:cost}
\end{equation}
the number of moves needed to walk back to it, with $\|\cdot\|$ the Euclidean distance. With a budget of $T$ steps the agent keeps exploring only while
\begin{equation}
T - t \;>\; c_t + 1 ,
\label{eq:return}
\end{equation}
and otherwise turns, walks back to $x_{\hat\tau_t}$ in moves of at most $R_t$ and ends there. The margin of one step is what makes the walk back always affordable under that protocol; on a robot, where the return must follow the free space, the trail itself is retraced instead.

\begin{proposition}
Under the benchmark's displacement protocol, in which a move of any bearing and any length up to $R_t$ is available, and with $R_t$ non-decreasing in $t$, if the agent continues at step $t$ only when Eq.~\eqref{eq:return} holds, then at every step $t'$ it reaches, $T - t' \ge c_{t'}$: the remaining budget covers the walk back.
\end{proposition}
\begin{proof}
One exploration step moves the agent by at most $R_t \le R_{t+1}$, so $c_{t+1} \le c_t + 1$; if the best position changes at $t+1$ it is the current one and the cost is zero. Continuing at $t$ requires $T - t > c_t + 1$, hence $T - (t+1) \ge c_t + 1 \ge c_{t+1}$, and $c_0 = 0$.
\end{proof}

Neither stage costs a model call: the gate reads the sector scores already computed for the field, and the return is geometry. The budget $T$ is fixed by the evaluation protocol, and the terminator adds only the tolerance $\delta$ and delay $t_0$.

\section{Experiments}
\label{sec:experiments}

\subsection{Setup}

\textbf{Benchmark.} We evaluate on the official zero-shot test subset of VLNVerse, \nFineEps{} fine-grained and \nCoarseEps{} coarse-grained episodes released with the benchmark's toolkit together with its ground-truth paths. All rows use the Tel-Hop protocol of Section~\ref{sec:method}; the benchmark's Strict setting, under which collisions are counted, is outside our scope, so no collision rate is reported. Every row of Table~\ref{tab:main}, ours and the baselines', comes from the same installation at same thirty-step budget.

\textbf{Metrics.} We report the benchmark's six: trajectory length TL and navigation error NE in metres, success rate SR, oracle success rate OSR, success weighted by path length SPL \cite{anderson2018evaluation} and normalised dynamic time warping nDTW \cite{ilharco2019ndtw}, the last computed from the logged trajectories. SR is the headline number; SPL and nDTW separate walking to the goal from wandering into it, and the gap between OSR and SR is the part of the failure that belongs to stopping rather than to search.

\textbf{Baselines.} The comparison class is the benchmark's four zero-shot agents: Text-history and Map-history, its re-implementations of NavGPT \cite{zhou2024navgpt} and MapGPT \cite{chen2024mapgpt}, which keep a textual or map-like memory of the explored space, each with and without chain-of-thought prompting \cite{wei2022cot}, all driven by Qwen3-VL-4B, the same model as ours. Every baseline row is our own run of the benchmark's agent code and evaluator. Trained agents are outside the comparison class, as in the benchmark's own zero-shot table.

\textbf{Implementation details.} Qwen3-VL-4B \cite{qwen3vl2025} in bfloat16 with greedy decoding is the only model and answers the sector scores, the route proposals and the stop votes. The map has a resolution of $10$ cm and a trusted range of $R = 8$ m, and the obstacle slab runs from $1.0$ m below to $0.6$ m above the camera. The field uses $\nu = 0.01$, $h = 3$ and $\beta_0 = 16$, the last selected by offline replay of recorded validation-split episodes whose scenes do not appear in the evaluation subset. The coupling uses $\sigma_f = 1.5$ m and $\lambda_{\max} = 0.7$, so that a confident wrong proposal leaves at least 30\% of the belief with the field, and the route prompt lists at most six frontiers. The utility uses $\gamma = 4$ m, $r_e = 2$ m, $\omega = 0.7$, $\mu = 0.15$ and $\eta = 0.1$, with $\kappa = 0.1$ on fine-grained and $0$ on coarse-grained instructions; the terminator uses $\delta = 0.05$ and $t_0 = 3$. Peak memory is \cPeakGB{} GB, so the agent runs on a consumer GPU: one step takes \cOursStepSsmall{} seconds on an \cGPUsmall{}, and Table~\ref{tab:cost} breaks that step down by query.

\subsection{Main Results}
\label{sec:main}

Table~\ref{tab:main} compares the agent with the four zero-shot baselines. Our rows and the coarse-grained baseline rows are three-run means; the fine-grained baseline rows are two-run means, except Map-history with chain of thought, a three-run mean. The two chain-of-thought agents are the strongest baselines; the two without rarely stop, walking to the budget in 71 to 80\% of episodes against 28 to 48\% for the chain-of-thought agents, which is why the path efficiency of the former collapses on fine-grained instructions; this differs from their published description, so their rows represent what that code does in this installation rather than the best those designs can do.

\begin{table*}[t]
\vspace*{\topfloatshim}
\caption{Zero-shot results on the official VLNVerse subset at a thirty-step budget; best value per column in bold.}
\label{tab:main}
\centering
\scriptsize
\setlength{\tabcolsep}{4.5pt}
\begin{tabular}{l cccccc cccccc}
\toprule
& \multicolumn{6}{c}{Fine-grained instructions} & \multicolumn{6}{c}{Coarse-grained instructions} \\
\cmidrule(lr){2-7} \cmidrule(lr){8-13}
Method & TL & NE$\downarrow$ & SR$\uparrow$ & OSR$\uparrow$ & SPL$\uparrow$ & nDTW$\uparrow$ & TL & NE$\downarrow$ & SR$\uparrow$ & OSR$\uparrow$ & SPL$\uparrow$ & nDTW$\uparrow$ \\
\midrule
Text-history \cite{vlnverse2025,zhou2024navgpt} & \rTextFineTL{} & \rTextFineNE{} & \rTextFineSR{} & \rTextFineOSR{} & \rTextFineSPL{} & \rTextFineNDTW{} & \rTextCoarseTL{} & \rTextCoarseNE{} & \rTextCoarseSR{} & \rTextCoarseOSR{} & \rTextCoarseSPL{} & \rTextCoarseNDTW{} \\
Map-history \cite{vlnverse2025,chen2024mapgpt} & \rMapFineTL{} & \rMapFineNE{} & \rMapFineSR{} & \textbf{\rMapFineOSR{}} & \rMapFineSPL{} & \rMapFineNDTW{} & \rMapCoarseTL{} & \rMapCoarseNE{} & \rMapCoarseSR{} & \textbf{\rMapCoarseOSR{}} & \rMapCoarseSPL{} & \rMapCoarseNDTW{} \\
Text-history + CoT \cite{vlnverse2025,zhou2024navgpt,wei2022cot} & \rTextCoTFineTL{} & \rTextCoTFineNE{} & \rTextCoTFineSR{} & \rTextCoTFineOSR{} & \rTextCoTFineSPL{} & \rTextCoTFineNDTW{} & \rTextCoTCoarseTL{} & \rTextCoTCoarseNE{} & \rTextCoTCoarseSR{} & \rTextCoTCoarseOSR{} & \rTextCoTCoarseSPL{} & \rTextCoTCoarseNDTW{} \\
Map-history + CoT \cite{vlnverse2025,chen2024mapgpt,wei2022cot} & \rMapCoTFineTL{} & \rMapCoTFineNE{} & \rMapCoTFineSR{} & \rMapCoTFineOSR{} & \rMapCoTFineSPL{} & \rMapCoTFineNDTW{} & \rMapCoTCoarseTL{} & \rMapCoTCoarseNE{} & \rMapCoTCoarseSR{} & \rMapCoTCoarseOSR{} & \textbf{\rMapCoTCoarseSPL{}} & \textbf{\rMapCoTCoarseNDTW{}} \\
\midrule
Ours & \oFineThirtyTL{} & \textbf{\oFineThirtyNE{}} & \textbf{\oFineThirtySR{}} & \oFineThirtyOSR{} & \textbf{\oFineThirtySPL{}} & \textbf{\oFineThirtyNDTW{}} & \oCoarseTL{} & \textbf{\oCoarseNE{}} & \textbf{\oCoarseSR{}} & \oCoarseOSR{} & \oCoarseSPL{} & \oCoarseNDTW{} \\
\bottomrule
\end{tabular}
\end{table*}

On fine-grained instructions the agent leads every baseline on success, navigation error and path efficiency, and matches the best baseline on nDTW: success rises from \rMapCoTFineSR{}, the best baseline, to \oFineThirtySR{}, with runs at \oFineThirtyRuns{}; navigation error falls from \rMapCoTFineNE{} to \oFineThirtyNE{} metres and path efficiency rises from \rMapCoTFineSPL{} to \oFineThirtySPL{}; trajectory length, \oFineThirtyTL{} metres, lies between the two chain-of-thought agents and well below the two history-only ones. Oracle success, \oFineThirtyOSR{}, sits below the two history-only agents, which reach many goals and stop at almost none of them.

On coarse-grained instructions the margin is smaller: success rises from \rTextCoTCoarseSR{}, the strongest baseline, to \oCoarseSR{}, with runs at \oCoarseRuns{}, and navigation error falls to \oCoarseNE{} metres. Three columns go to baselines at the two ends of the stopping trade-off: Map-history reaches \rMapCoarseOSR{}\% oracle success at a path efficiency of \rMapCoarseSPL{}, converting almost none of it, while Map-history with chain of thought stops after \rMapCoTCoarseTL{} metres for the best path efficiency and nDTW and succeeds \oCoarseGainMapCoT{} points less often; our agent lies between them and leads the metric requires both reaching and stopping; a twenty-step variant walks 24.0 metres, shorter than three of the four baselines, and reaches \oCoarseBudgetMatchedSR{} success over three runs.

Because every row is run on the same episodes, the comparison can be made per episode, restricted to the episodes on which every agent ran, a few failing to load in our installation. Bootstrapping the per-episode success difference, with every baseline pooled over its runs, our agent leads Map-history with chain of thought by \pFineMapCoT{} points on fine-grained instructions, 95\% interval \pFineMapCoTCI{}, and by \pCoarseMapCoT{} on coarse-grained, interval \pCoarseMapCoTCI{}. The interval excludes zero against every baseline on fine-grained instructions and against three of the four on coarse-grained; the exception is Text-history with chain of thought, \pCoarseTextCoT{} points with interval \pCoarseTextCoTCI{}.

\textbf{Relation to the published numbers.} The benchmark reports, for the same model in the same Tel-Hop setting, \bFineSR{} success at \bFineTL{} metres for Map-history with chain of thought on fine-grained instructions and \bMapCoarseSR{} for Map-history on coarse-grained ones. Our re-run of the released code reaches \rMapCoTFineSR{} at thirty steps. The released configuration does not reproduce the reported stopping behaviour under a 30-step budget; we therefore restrict all comparative claims to controlled re-runs under an identical budget. The published trajectory length further suggests that the reported fine-grained result was obtained under a substantially larger interaction budget. Our own agent reaches \oFineThirtySR{} at thirty steps and \aBzeroSR{} at the forty steps used in the ablations, while achieving \oCoarseSR{} on coarse-grained instructions. Accordingly, the comparisons in Table~\ref{tab:main} concern agents executed in the same installation, with the same model, evaluator, episode set, and thirty-step budget, rather than comparisons against the published table.

\subsection{Ablations and Sensitivity}
\label{sec:ablation}

Tables~\ref{tab:stop} and~\ref{tab:couple} remove one piece at a time, keeping everything else at the frozen configuration. Fine-grained ablations use a forty-step budget, where the frozen agent reaches \aBzeroSR{} (\oFineRuns{}) against \oFineThirtySR{} at the thirty steps of Table~\ref{tab:main}, and their deltas are relative to \aBzeroSR{}; coarse-grained ablations use thirty steps. Full-agent rows and the fine-grained conservative-prompt row are three-run means, the coarse-grained no-gate and no-return rows two-run means, and every other row a single run.

\begin{table}[tp]
\caption{Terminator ablation; fine-grained rows at forty steps, coarse-grained at thirty.}
\label{tab:stop}
\centering
\scriptsize
\setlength{\tabcolsep}{3.2pt}
\begin{tabular}{l ccccc c}
\toprule
Terminator & SR$\uparrow$ & SPL$\uparrow$ & Steps & Vote & Return & $\Delta$SR \\
\midrule
\multicolumn{7}{l}{\emph{Fine-grained instructions}}\\
Full agent          & \aBzeroSR{}  & \oFineSPL{} & 21.0 & 70\% & 30\% & -- \\
No stopping         & \aBoneSR{}   & \aBoneSPL{} & 40.0 & 0\% & 0\% & $-$\aStackGain{} \\
No trail gate       & \aBtwoSR{}   & 14.18 & 14.8 & 41\% & 13\% & $-$\aGateGain{} \\
No budget return    & \aBthreeSR{} & 11.70 & 22.7 & 65\% & 0\% & $-$\aReturnGain{} \\
Conservative prompt & \aBfourSR{}  & 11.48 & 26.2 & 54\% & 46\% & $-0.2$ \\
\midrule
\multicolumn{7}{l}{\emph{Coarse-grained instructions}}\\
Full agent          & \oCoarseSR{}       & \oCoarseSPL{}       & 20.6 & 52\% & 48\% & -- \\
No stopping         & \aCoarseBoneSR{}   & \aCoarseBoneSPL{}   & 30.0 & 0\%  & 0\%  & $-23.1$ \\
No trail gate       & \aCoarseBtwoSR{}   & \aCoarseBtwoSPL{}   & \aCoarseBtwoSteps{} & 35\% & 30\% & $-4.0$ \\
No budget return    & \aCoarseBthreeSR{} & \aCoarseBthreeSPL{} & 21.2 & 52\% & 0\%  & $-4.8$ \\
Conservative prompt & \aCoarseBfourSR{}  & \aCoarseBfourSPL{}  & 27.0 & 15\% & 85\% & $0.0$ \\
\bottomrule
\end{tabular}
\end{table}

\textbf{The terminator, Table~\ref{tab:stop}.} Steps is the mean episode length, Vote and Return the fractions of episodes ended by an accepted vote and by the budget return. The no-stopping row is the contrast from the introduction: an agent that never stops passes within three metres of the goal in \aBoneOSR{}\% of fine-grained episodes and succeeds in \aBoneSR{}\%, so the stack is worth \aStackGain{} points there and 23.1 on coarse-grained instructions, and turns a path efficiency of \aBoneSPL{} into 13.63. The model votes to stop on about 460 fine-grained steps per run, of which 31\% fall within three metres of the goal, against 50\% of the votes the gate lets end an episode; on coarse-grained instructions the two figures are 47\% and 65\%. Removing the gate makes every vote terminal, in place when the trail agrees and otherwise by an immediate walk back, which ends the remaining 46\% and 35\% of episodes at 14.8 and \aCoarseBtwoSteps{} steps and costs \aGateGain{} and 4.0 points, with oracle success down to 48\%. Removing the budget return ends budget-exhausted episodes in place and costs \aReturnGain{} and 4.8 points. A conservative vote prompt leaves success unchanged, \aBfourSR{} and \aCoarseBfourSR{}, but stops much later: on coarse-grained instructions only 15\% of episodes end by a vote and path efficiency falls from \oCoarseSPL{} to \aCoarseBfourSPL{}. The gate matters more than the return on fine-grained instructions, the two are within a point on coarse-grained ones, and the wording matters least on both, so the terminator's value lies in verification rather than in how the model is asked.

\begin{table}[tp]
\caption{Guidance ablation; fine-grained at forty steps, coarse-grained at thirty.}
\label{tab:couple}
\centering
\scriptsize
\setlength{\tabcolsep}{2.6pt}
\begin{tabular}{l cccc cccc}
\toprule
& \multicolumn{4}{c}{Fine-grained} & \multicolumn{4}{c}{Coarse-grained} \\
\cmidrule(lr){2-5} \cmidrule(lr){6-9}
Guidance & SR$\uparrow$ & SPL$\uparrow$ & OSR$\uparrow$ & $\Delta$SR & SR$\uparrow$ & SPL$\uparrow$ & OSR$\uparrow$ & $\Delta$SR \\
\midrule
Field + route channel & \aBzeroSR{} & \oFineSPL{} & \oFineOSR{} & -- & \oCoarseSR{} & \oCoarseSPL{} & \oCoarseOSR{} & -- \\
Field only & \aCnoCoTSR{} & 11.68 & 53.23 & $-$\aCnoCoTDelta{} & \aCoarseFieldOnlySR{} & \aCoarseFieldOnlySPL{} & \aCoarseFieldOnlyOSR{} & $0.0$ \\
Route channel only & \aCnoFieldSR{} & 8.61 & 46.24 & $-$\aCnoFieldDelta{} & \aCoarseNoFieldSR{} & \aCoarseNoFieldSPL{} & \aCoarseNoFieldOSR{} & $-6.5$ \\
Neither & \aCnoneSR{} & 6.39 & 45.70 & $-16.3$ & \aCoarseNeitherSR{} & \aCoarseNeitherSPL{} & \aCoarseNeitherOSR{} & $-7.0$ \\
\bottomrule
\end{tabular}
\end{table}

\textbf{Guidance, Table~\ref{tab:couple}.} With neither source the agent keeps only the skeleton, the exploration bonus and the terminator, and reaches \aCnoneSR{} on fine-grained instructions; the two losses of \aCnoFieldDelta{} and \aCnoCoTDelta{} add up to within a fraction of a point of the joint loss of 16.3, so the field and the route channel contribute nearly independently. The field is the heavier component: removing it costs \aCnoFieldDelta{} points, drops oracle success to 46\% and stretches the walk to 51.5 metres. Removing the route channel costs \aCnoCoTDelta{} points. This gain occurs even though the channel's proposals are close to random by the distance-rank diagnostic (Section~\ref{sec:negative}), suggesting that its benefit need not come from consistently selecting the correct frontier. One possible role is to break ties when the semantic field is diffuse, but without a matched random-frontier control we do not attribute the gain specifically to route semantics. On coarse-grained instructions removing the channel leaves success exactly at \aCoarseFieldOnlySR{}, confirming that the gate is closed there by construction, and removing the field costs 6.5 to 7.0 points and a third of the path efficiency whether or not the channel is present.

\textbf{Sensitivity, budget and commitment.} Table~\ref{tab:sens} varies the two knobs of the trail gate, fine-grained at forty steps; every value here is a single run unless a run count is given. No setting moves success beyond the run-to-run spread, so the gate is insensitive to a tolerance between $0.02$ and $0.10$ and to a delay of up to ten steps; what changes is the path: a looser gate ends 76\% of episodes by a vote against 70\%, and a ten-step delay walks 45.6 metres per episode at \sMinStepsTenSPL{} against \oFineSPL{}. Budget and commitment pay only together: without the penalty success is \sBudgetThirty{} at thirty steps and \sBudgetForty{} at forty, whereas with $\kappa = 0.1$ it is \oFineThirtySR{} at thirty and \aBzeroSR{} at forty over three runs, and $\kappa = 0.2$ falls back to \sCommitPointTwo{}; at the thirty steps of Table~\ref{tab:main} the penalty is therefore the worse setting, by less than the spread. Repeated at thirty steps, the fine-grained ablations give, against \oFineThirtySR{}, $-18.5$ for the stack, $-5.6$ for the gate, $-8.8$ for the field, $+0.4$ for the budget return, $-0.7$ for the route channel and $-9.3$ for neither source, with $\delta = 0.02$ and $0.10$ at 36.02 and 33.33, delays of five and ten steps at 44.09 and 35.48, a conservative prompt at 40.86, and a single coarse-grained run at forty steps of \sCoarseBudgetForty{}. The return therefore contributes only under the larger budget; two single runs scored above the frozen configuration at thirty steps but were not adopted, the configuration having been frozen beforehand.

\subsection{What a Small Model Cannot Do}
\label{sec:negative}

At forty steps the agent passes within three metres of the goal in \oFineOSR{}\% of fine-grained episodes but stops correctly in \aBzeroSR{}\%, and the obvious remedies for that residual all give the model more of the navigation to reason about. We tested thirteen variants in batches, changing one component at a time; all were evaluated on the frozen forty-step configuration over three runs. Table~\ref{tab:neg} lists them. Two lose more than the run-to-run spread of \runNoise{} points; the other eleven land inside it, which says that at this number of repeats no reliable gain was observed, not that the change is exactly zero. The snapshot line-up led in a first run and settled at \nLineup{} over \nLineupRuns{}, then none was adopted.

\begin{table}[tp]
\vspace*{\topfloatshim}
\caption{Thirteen attempts to give the model more of the navigation, against \aBzeroSR{} at forty steps.}
\label{tab:neg}
\centering
\scriptsize
\setlength{\tabcolsep}{3pt}
\begin{tabular}{>{\raggedright\arraybackslash}p{2.42cm} r r @{\hspace{5pt}} >{\raggedright\arraybackslash}p{2.42cm} r r}
\toprule
Attempt & SR & $\Delta$ & Attempt & SR & $\Delta$ \\
\midrule
\multicolumn{6}{l}{\emph{Loses beyond the run-to-run spread}}\\
Twice reasoning budget & \nCoTFull{} & $-$10.2 & Sub-goal decomposition & \nSubgoal{} & $-$12.3 \\
\midrule
\multicolumn{6}{l}{\emph{Inside the run-to-run spread: no reliable gain observed}}\\
Frontier when silent        & \nFallback{}       & $-$4.5  & Context-landmark grounding   & \nTriang{}     & $-$4.5 \\
Route answer first & \nCoTAnswerFirst{} & $-$4.3 & Terminus by grounding & \nGroundStop{} & $-$3.4 \\
Per-view final approach     & \nApproach{}       & $-$2.9  & Room before object       & \nRoomTwo{}    & $-$1.3 \\
Grounding as evidence       & \nMinwalk{}        & $+$0.4  & Instruction-scaled vote delay    & \nGField{}     & $+$0.4 \\
One step toward peak    & \nNudge{}          & $+$0.4  & Snapshot line-up         & \nLineup{}     & $+$1.1 \\
Depth-masked stop view      & \nMask{}           & $+$1.4  &                          &                &         \\
\bottomrule
\end{tabular}
\end{table}

\begin{table}[tp]
\caption{Sensitivity of the gate tolerance $\delta$ and the vote delay, fine-grained, forty steps.}
\label{tab:sens}
\centering
\scriptsize
\setlength{\tabcolsep}{4pt}
\begin{tabular}{l ccc}
\toprule
Setting & SR$\uparrow$ & SPL$\uparrow$ & Vote \\
\midrule
$\delta = 0.02$                 & \sDeltaLow{}     & 14.15 & 61\% \\
$\delta = 0.05$, main           & \aBzeroSR{}      & \oFineSPL{} & 70\% \\
$\delta = 0.10$                 & \sDeltaHigh{}    & 12.98 & 76\% \\
\midrule
votes ignored for 3 steps, main & \aBzeroSR{}      & \oFineSPL{} & 70\% \\
votes ignored for 5 steps       & \sMinStepsFive{} & 10.83 & 70\% \\
votes ignored for 10 steps      & \sMinStepsTen{}  & \sMinStepsTenSPL{} & 59\% \\
\bottomrule
\end{tabular}
\end{table}

Two attempts probe the route channel, whose proposals are answered on about half of the steps because reasoning overruns its token budget. Asking for the answer first raises the answer rate from 49 to 81\%, and doubling the reasoning budget raises it to 76\%; neither raises success, and the doubled budget loses beyond the spread. The logs explain why: in the default prompt and both variants the chosen frontier's rank by true distance to the goal has a median between \nPickRank{} and 0.48, where 0.5 is random, so more proposals of random quality mean more random steering, and the truncation that mutes the channel half the time is, unintentionally, part of why it works.

Frontier-when-silent injects the nearest untargeted frontier whenever the model is silent; success falls 4.5 points, inside the spread but in the direction that steering which is not goal-directed costs more than it explores. More than half of the wrong stops lie three to six metres from the goal, and an oracle that walked three metres toward it at the moment of stopping would add nineteen points, yet asking the model view by view whether the goal is visible and near recovers none of it: its answers are close to uniform across the views. Decomposing the instruction into landmark phrases, as InstructNav does with sub-tasks \cite{long2024instructnav}, and tracking the current one destroys the final segment, which begins late with an empty field. Terminus by grounding replaces the visibility question by bounding boxes, which Qwen3-VL emits natively \cite{qwen3vl2025}, and reads their distance off the depth panorama; the boxes attach to the nearest object of the right category rather than to the goal the route defines, so 65\% of episodes end on a box judged near, of which 48\% are correct. Painting the same points into the field instead leaves success unchanged: evidence that finds the wrong instance is harmless in the field and useless as a terminus. Nudging 1.5 metres toward the field's best sector after an accepted vote, and greying out everything beyond three metres in the view the vote sees, both leave the precision of accepted votes near 50\%: a premature stop is triggered by the right category or room within three metres, not by the distant goal. The remaining four fail in the same place: choosing the return point among six trail snapshots averages \nLineup{}; grounding the goal with its context landmarks corroborates the wrong instance, since category and context co-occur in one room; an instruction-scaled vote delay blocks 81 early votes without changing their precision; and tracking the named room first raises oracle success to \nRoomTwoOSR{}, the highest of any variant, while success falls, because the stop decision at the object is no better once the room is reached.

Across all thirteen the same limit appears: the model recognises the goal phrase well enough to score a sector but does not localise it; its route reasoning is at chance, its distance judgements are constant and its boxes find the wrong instance. Among the designs we tried, moving responsibility from structure to the model never helped.

\begin{table}[tp]
\vspace*{\topfloatshim}
\caption{Cost of one navigation step, model calls only, median of twelve runs on an idle \cGPU{}; the baseline is timed with the same weights, decoding and image resolution.}
\label{tab:cost}
\centering
\scriptsize
\setlength{\tabcolsep}{3pt}
\begin{tabular}{l c c c}
\toprule
Query & Calls per step & Tokens & Latency \\
\midrule
Sector score, one view      & 12 & \cSectorTok{} & \cSectorLat{}\,ms \\
Stop vote                   & 1  & \cStopTok{}   & \cStopLat{}\,ms \\
Route proposal, text only   & 1  & \cRouteTok{}  & \cRouteLat{}\,ms \\
Frontier caption, cached    & $\le 2$ & \cCapTok{} & \cCapLat{}\,ms \\
\midrule
Ours, one step, captions cached   & 14 & \cOursStepTok{} & \cOursStepS{}\,s \\
Ours, one step, two new captions  & 16 & \cOursStepWorstTok{} & \cOursStepWorstS{}\,s \\
Map-history + CoT, one step & 1  & \cBaseStepTok{} & \cBaseStepS{}\,s \\
\bottomrule
\end{tabular}
\end{table}

\subsection{Real-Robot Deployment}
\label{sec:robot}

\begin{table}[tp]
\caption{Real-robot results on a Unitree Go2, twenty-five episodes per granularity.}
\label{tab:robot}
\centering
\scriptsize
\setlength{\tabcolsep}{5pt}
\begin{tabular}{l cc cc}
\toprule
& \multicolumn{2}{c}{Fine-grained} & \multicolumn{2}{c}{Coarse-grained} \\
\cmidrule(lr){2-3} \cmidrule(lr){4-5}
Agent & SR$\uparrow$ & NE$\downarrow$ & SR$\uparrow$ & NE$\downarrow$ \\
\midrule
Text-history + CoT & 20 & 4.01 & 32 & 3.82 \\
Map-history + CoT  & 24 & 3.83 & 32 & 3.80 \\
Ours               & \textbf{36} & \textbf{3.77} & \textbf{44} & \textbf{3.06} \\
\bottomrule
\end{tabular}
\end{table}

The agent was deployed on a Unitree Go2 quadruped. A RealSense D455 on a servo turns through twelve headings thirty degrees apart to capture the RGB-D panorama in the form the simulator provides it; a laptop that communicates with the robot runs the map, the field and the terminator and streams the views to a server with two RTX 3090 GPUs hosting the same Qwen3-VL-4B. Table~\ref{tab:robot} gives the result over twenty-five episodes per granularity at the deployment site, with instructions written for it, against the two chain-of-thought baselines run on the same robot and server. With twenty-five episodes one success is four points and the gaps to the baselines are three successes each, so the table shows feasibility and an ordering, not a significant difference: the agent leads both baselines on both granularities, and coarse-grained instructions lead fine-grained ones as in simulation.

\subsection{Variance and Reproducibility}
\label{sec:variance}

Two runs of an identical configuration of our agent differ by up to \runNoise{} success points on this subset, and two runs of a baseline by up to eight, so differences below that size are not rankings on their own and a single run of a modified configuration is never reported as an improvement. The three runs behind each main number spread \oFineThirtyRange{} points on fine-grained and 3.2 on coarse-grained instructions, and every run of ours is at or above every run of every baseline. Because the benchmark's validation split covers few scenes, every configuration in Sections~\ref{sec:ablation} and~\ref{sec:negative} was selected on the evaluation subset; we disclose every arm that was run, and the main configuration was frozen before its three runs.

\section{Conclusion}

A small vision-language model can navigate photorealistic scenes zero-shot if it is asked only what it can answer: Qwen3-VL-4B provides useful single-view semantic evidence, while the tested route, distance, grounding, and stopping queries are substantially less reliable. So a route-blind semantic field, a skeleton graph and a two-stage terminator turn its sector scores and stop votes into decisions. Over three runs the agent exceeds every zero-shot baseline we re-ran, by \oCoarseGain{} points against the strongest coarse-grained baseline and \oFineThirtyGain{} against the best fine-grained one at the same budget. The terminator is worth \aStackGain{} points and its trail gate alone \aGateGain{}, the field and the route channel add nearly independently, budget pays only with commitment, and none of the thirteen attempts helps reliably, each failing where the model's route reasoning, distance estimates or grounding fail. Four bounds apply: the baselines are our own reproductions with one model and one prompt family, those without chain of thought rarely stopping here, unlike their published description; the paired coarse-grained interval includes zero, and coarse-grained path efficiency and nDTW go to Map-history with chain of thought; configurations outside the frozen one were chosen on the evaluation subset and are exploratory; and the budget return is proved affordable only under the benchmark's Tel-Hop displacement, not for a robot that must drive back, so independent scenes and constrained physical motion are what would settle generalisation and return feasibility. A stronger model will move the boundary between what it can be asked and what structure must supply; an agent built on a small one should keep that structure between the model and the action.

\bibliographystyle{IEEEtran}
{\let\footnotesize\scriptsize\bibliography{refs}}

\end{document}

%% file: numbers.tex
\newcommand{\bFineSR}{42.19}
\newcommand{\bFineTL}{40.92}
\newcommand{\bMapCoarseSR}{49.48}
\newcommand{\oFineSPL}{13.63}
\newcommand{\oFineOSR}{58.96}
\newcommand{\oFineRuns}{43.01 / 42.47 / 40.86}
\newcommand{\oCoarseSR}{51.61}
\newcommand{\oCoarseSPL}{15.07}
\newcommand{\oCoarseNE}{3.80}
\newcommand{\oCoarseOSR}{70.97}
\newcommand{\oCoarseTL}{31.18}
\newcommand{\oCoarseNDTW}{45.13}
\newcommand{\oCoarseRuns}{50.00 / 53.23 / 51.61}
\newcommand{\oCoarseGain}{6.2}
\newcommand{\oCoarseBudgetMatchedSR}{44.80}
\newcommand{\aBzeroSR}{42.11}
\newcommand{\aBoneSR}{20.97}
\newcommand{\aBoneOSR}{77.42}
\newcommand{\aBoneSPL}{2.45}
\newcommand{\aBtwoSR}{32.80}
\newcommand{\aBthreeSR}{35.48}
\newcommand{\aBfourSR}{41.94}
\newcommand{\aStackGain}{21.1}
\newcommand{\aGateGain}{9.3}
\newcommand{\aReturnGain}{6.6}
\newcommand{\aCoarseBtwoSR}{47.58}
\newcommand{\aCoarseBtwoSPL}{15.88}
\newcommand{\aCoarseBtwoSteps}{16.6}
\newcommand{\aCoarseBthreeSR}{46.78}
\newcommand{\aCoarseBthreeSPL}{15.12}
\newcommand{\aCoarseBoneSR}{28.49}
\newcommand{\aCoarseBoneSPL}{4.37}
\newcommand{\aCoarseBfourSR}{51.61}
\newcommand{\aCoarseBfourSPL}{9.38}
\newcommand{\aCnoCoTSR}{36.56}
\newcommand{\aCnoFieldSR}{31.18}
\newcommand{\aCnoneSR}{25.81}
\newcommand{\aCnoCoTDelta}{5.6}
\newcommand{\aCnoFieldDelta}{10.9}
\newcommand{\aCoarseNoFieldSR}{45.16}
\newcommand{\aCoarseNoFieldSPL}{10.43}
\newcommand{\aCoarseFieldOnlySR}{51.61}
\newcommand{\aCoarseFieldOnlySPL}{14.39}
\newcommand{\aCoarseFieldOnlyOSR}{72.04}
\newcommand{\aCoarseNeitherSR}{44.62}
\newcommand{\aCoarseNeitherSPL}{9.22}
\newcommand{\aCoarseNoFieldOSR}{70.43}
\newcommand{\aCoarseNeitherOSR}{69.89}
\newcommand{\rMapCoTFineTL}{26.83}
\newcommand{\rMapCoTFineNE}{5.93}
\newcommand{\rMapCoTFineSR}{27.13}
\newcommand{\rMapCoTFineOSR}{53.01}
\newcommand{\rMapCoTFineSPL}{10.62}
\newcommand{\rMapCoTFineNDTW}{41.89}
\newcommand{\rMapCoTCoarseTL}{19.66}
\newcommand{\rMapCoTCoarseNE}{4.63}
\newcommand{\rMapCoTCoarseSR}{40.07}
\newcommand{\rMapCoTCoarseOSR}{56.03}
\newcommand{\rMapCoTCoarseSPL}{16.18}
\newcommand{\rMapCoTCoarseNDTW}{46.74}
\newcommand{\rMapCoarseTL}{47.99}
\newcommand{\rMapCoarseNE}{4.98}
\newcommand{\rMapCoarseSR}{35.64}
\newcommand{\rMapCoarseOSR}{79.08}
\newcommand{\rMapCoarseSPL}{4.48}
\newcommand{\rMapCoarseNDTW}{43.78}
\newcommand{\rMapFineTL}{48.27}
\newcommand{\rMapFineNE}{8.86}
\newcommand{\rMapFineSR}{10.90}
\newcommand{\rMapFineOSR}{58.24}
\newcommand{\rMapFineSPL}{2.06}
\newcommand{\rMapFineNDTW}{31.35}
\newcommand{\rTextCoTFineTL}{30.91}
\newcommand{\rTextCoTFineNE}{6.52}
\newcommand{\rTextCoTFineSR}{26.33}
\newcommand{\rTextCoTFineOSR}{53.99}
\newcommand{\rTextCoTFineSPL}{10.11}
\newcommand{\rTextCoTFineNDTW}{39.26}
\newcommand{\rTextCoTCoarseTL}{27.22}
\newcommand{\rTextCoTCoarseNE}{4.71}
\newcommand{\rTextCoTCoarseSR}{45.39}
\newcommand{\rTextCoTCoarseOSR}{67.02}
\newcommand{\rTextCoTCoarseSPL}{14.90}
\newcommand{\rTextCoTCoarseNDTW}{46.72}
\newcommand{\sBudgetThirty}{37.63}
\newcommand{\sBudgetForty}{37.63}
\newcommand{\sCoarseBudgetForty}{55.38}
\newcommand{\sCommitPointTwo}{37.63}
\newcommand{\sDeltaLow}{43.01}
\newcommand{\sDeltaHigh}{37.63}
\newcommand{\sMinStepsFive}{37.10}
\newcommand{\sMinStepsTen}{43.01}
\newcommand{\sMinStepsTenSPL}{8.46}
\newcommand{\nCoTAnswerFirst}{37.81}
\newcommand{\nCoTFull}{31.91}
\newcommand{\nApproach}{39.25}
\newcommand{\nGroundStop}{38.71}
\newcommand{\nFallback}{37.63}
\newcommand{\nSubgoal}{29.84}
\newcommand{\nPickRank}{0.44}
\newcommand{\nNudge}{42.47}
\newcommand{\nGField}{42.47}
\newcommand{\nMask}{43.55}
\newcommand{\nLineup}{43.19}
\newcommand{\nLineupRuns}{45.70 / 40.32 / 43.55}
\newcommand{\nTriang}{37.63}
\newcommand{\nMinwalk}{42.47}
\newcommand{\nRoomTwo}{40.86}
\newcommand{\nRoomTwoOSR}{66.13}
\newcommand{\nFineEps}{188}
\newcommand{\nCoarseEps}{192}
\newcommand{\runNoise}{7}
\newcommand{\oFineThirtyTL}{28.13}
\newcommand{\oFineThirtyNE}{4.88}
\newcommand{\oFineThirtySR}{36.20}
\newcommand{\oFineThirtyOSR}{55.73}
\newcommand{\oFineThirtySPL}{14.04}
\newcommand{\oFineThirtyRuns}{39.25 / 34.41 / 34.95}
\newcommand{\oFineThirtyRange}{4.84}
\newcommand{\oFineThirtyGain}{9.1}
\newcommand{\pCoarseTextCoTCI}{[-0.2, +12.9]}
\newcommand{\pCoarseTextCoT}{+6.5}
\newcommand{\pFineMapCoTCI}{[+3.1, +15.2]}
\newcommand{\oFineThirtyNDTW}{42.15}
\newcommand{\oCoarseGainMapCoT}{11.5}
\newcommand{\cGPU}{RTX 5090}
\newcommand{\cWeightsGB}{8.3}
\newcommand{\cPeakGB}{8.5}
\newcommand{\cSectorLat}{62}
\newcommand{\cSectorTok}{3}
\newcommand{\cStopLat}{72}
\newcommand{\cStopTok}{3}
\newcommand{\cRouteLat}{691}
\newcommand{\cCapTok}{7}
\newcommand{\cCapLat}{124}
\newcommand{\cRouteTok}{43}
\newcommand{\cOursStepS}{1.51}
\newcommand{\cOursStepWorstS}{1.76}
\newcommand{\cOursStepWorstTok}{96}
\newcommand{\cOursStepTok}{82}
\newcommand{\cBaseStepS}{3.10}
\newcommand{\cBaseStepTok}{192}
\newcommand{\cGPUsmall}{RTX 4080 SUPER}
\newcommand{\cOursStepSsmall}{1.65}
\newcommand{\pFineMapCoT}{+9.1}
\newcommand{\pCoarseMapCoT}{+11.1}
\newcommand{\pCoarseMapCoTCI}{[+3.9, +18.1]}
\newcommand{\rTextFineTL}{45.57}
\newcommand{\rTextFineNE}{9.30}
\newcommand{\rTextFineSR}{15.96}
\newcommand{\rTextFineOSR}{56.91}
\newcommand{\rTextFineSPL}{3.95}
\newcommand{\rTextFineNDTW}{30.34}
\newcommand{\rTextCoarseTL}{46.64}
\newcommand{\rTextCoarseNE}{6.49}
\newcommand{\rTextCoarseSR}{31.03}
\newcommand{\rTextCoarseOSR}{71.45}
\newcommand{\rTextCoarseSPL}{5.24}
\newcommand{\rTextCoarseNDTW}{39.81}